%% file: arxiv.tex
\def\year{2022}\relax
\documentclass[letterpaper]{article} % DO NOT CHANGE THIS
\usepackage{aaai22}  % DO NOT CHANGE THIS
\usepackage{times}  % DO NOT CHANGE THIS
\usepackage{helvet}  % DO NOT CHANGE THIS
\usepackage{courier}  % DO NOT CHANGE THIS
\usepackage[hyphens]{url}  % DO NOT CHANGE THIS
\usepackage{graphicx} % DO NOT CHANGE THIS
\usepackage{natbib}  % DO NOT CHANGE THIS AND DO NOT ADD ANY OPTIONS TO IT
\usepackage{caption} % DO NOT CHANGE THIS AND DO NOT ADD ANY OPTIONS TO IT
\DeclareCaptionStyle{ruled}{labelfont=normalfont,labelsep=colon,strut=off} % DO NOT CHANGE THIS
\usepackage{amsmath}
\usepackage{amsthm}
\usepackage{booktabs} 
\usepackage{amsfonts}
\usepackage{mathtools}
\usepackage{verbatim}
\usepackage{url}
\usepackage{caption}
\usepackage{enumerate}
\usepackage{xcolor}
\usepackage{nicefrac}
\usepackage{microtype}
\usepackage{bm}
\usepackage{amssymb}
\usepackage{adjustbox}
\usepackage{listings}

\RequirePackage{algorithm}
\RequirePackage{algorithmic}

\title{Max-Margin Contrastive Learning}
\author{%
  Anshul Shah$^1$\thanks{Equal Contribution.} \qquad Suvrit Sra$^2$ \qquad Rama Chellappa$^1$ \qquad  Anoop Cherian$^{3*}$\thanks{Corresponding Author.}\\
}
\affiliations{$^1$Johns Hopkins University, Baltimore, MD\\
  $^2$Massachusetts Institute of Technology, Cambridge, MA\\
  $^3$Mitsubishi Electric Research Labs, Cambridge, MA \\
  \texttt{\{ashah95, rchella4\}@jhu.edu\quad  suvrit@mit.edu\quad  cherian@merl.com}}

\usepackage{bibentry}
\input{AAAI2022/definitions}

\begin{document}

\maketitle

\begin{abstract}
Standard contrastive learning approaches usually require a large number of negatives for effective unsupervised learning and often exhibit slow convergence. We suspect this behavior is due to the suboptimal selection of negatives used for offering contrast to the positives. We counter this difficulty by taking inspiration from support vector machines (SVMs) to present max-margin contrastive learning (MMCL). Our approach selects negatives as the sparse support vectors obtained via a quadratic optimization problem, and contrastiveness is enforced by maximizing the decision margin. As SVM optimization can be computationally demanding, especially in an end-to-end setting, we present simplifications that alleviate the computational burden. We validate our approach on standard vision benchmark datasets, demonstrating better performance in unsupervised representation learning over state-of-the-art, while having better empirical convergence properties. 
\end{abstract}

\section{Introduction}
Learning effective data representations is crucial to the success of any machine learning model. Recent years have seen a surge in algorithms for unsupervised representation learning that leverage the vast amounts of unlabeled data~\cite{chen2020simple,gidaris2018unsupervised,lee2017unsupervised,zhang2019aet,zhan2020online}.  In such algorithms, an auxiliary learning objective is typically designed to produce generalizable representations that capture some higher-order properties of the data. The assumption is that such properties could potentially be useful in (supervised) downstream tasks, which may have fewer annotated training samples. For example, in~\cite{noroozi2016unsupervised,santa2018visual}, the pre-text task is to solve patch jigsaw puzzles, so that the representations learned could potentially capture the natural semantic structure of images. Other popular auxiliary objectives include video frame prediction~\cite{oord2018representation}, image coloring~\cite{zhang2016colorful}, and deep clustering~\cite{caron2018deep}, to name a few.

Among the auxiliary objectives typically used for representation learning, one that has gained significant momentum recently is that of contrastive learning, which is a variant of the standard noise-constrastive estimation (NCE)~\cite{gutmann2010noise} procedure. In NCE, the goal is to learn data distributions by classifying the unlabeled data against random noise. However, recently developed contrastive learning methods learn representations by designing objectives that capture data invariances. Specifically, instead of using random noise as in NCE, these methods transform data samples to \emph{sets of samples}, each set consisting of transformed variants of a sample, and the auxiliary task is to classify one set (positives) against the rest (negatives). Surprisingly, even by using simple data transformations, such as color jittering, image cropping, or rotations, these methods are able to learn superior and generalizable representations, sometimes even outperforming supervised learning algorithms in downstream tasks (e.g., CMC~\cite{tian2019contrastive}, MoCo~\cite{chen2020improved,he2020momentum}, SimCLR~\cite{chen2020simple}, and BYOL~\cite{grill2020bootstrap}).

\begin{figure}
        \centering
        \includegraphics[width=\linewidth]{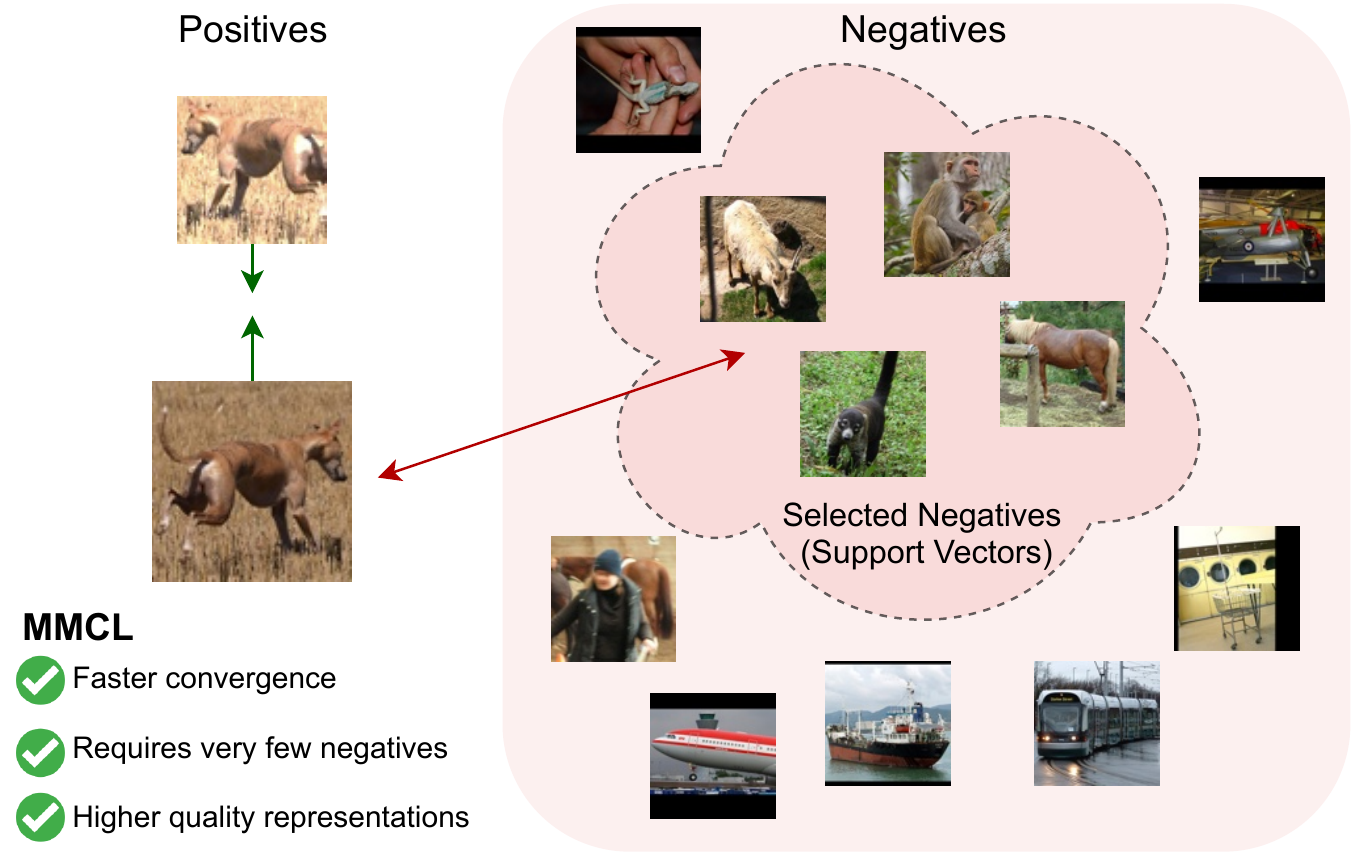}
        \caption{An illustration of our \emph{Max-Margin Contrastive Learning} framework. For every positive example, we compute a weighted subset of (hard) negatives via computing a discriminative hyperplane by solving an SVM objective. This hyperplane is then used in learning to maximize the similarity between the representations of the positives and minimize the similarity between the representations of the positives against the negatives. The negatives in the figure are actual ones selected by our scheme for the respective positive.}
        \vspace*{-0.5cm}
        \label{fig: teaser_fig}
    \end{figure}
    
Typically, contrastive learning methods use the NCE-loss
for the learning objective, which is usually a logistic
classifier separating the positives from the negatives. However, as is often found in NCE algorithms, the negatives should be \emph{close} in distribution to the positives for the learned representations to be useful -- a criteria that often demands a large number of negatives in practice (e.g., 16K in SimCLR~\cite{chen2020simple}). Further, standard contrastive learning approaches make the implicit assumption that the positives and negatives belong to distinct classes in the downstream task~\cite{arora2019theoretical}. This requirement is hard to enforce in an unsupervised training regime and defying this assumption may hurt the downstream performance due to beneficial discriminative cues being ignored.

In this paper, we explore alternative formulations for contrastive learning beyond the standard logistic classifier. Rather than contrasting the positive samples against all the negatives in a batch, our key insight is to design an objective that: (i) selects a suitable subset of  negatives to be contrasted against, and (ii) provides a means to  relax the effect of false negatives on the learned representations. Fig.~\ref{fig: teaser_fig} presents an overview of the idea. A natural objective in this regard is the classical support vector machine (SVM), which produces a discriminative hyperplane with the maximum margin separating the positives from the negatives. Inspired by SVMs, we propose a novel objective, \emph{max-margin contrastive learning} (MMCL), to learn data representations that maximizes the SVM decision margin. MMCL brings in several benefits to representation learning. For example, the \emph{kernel trick} allows for the use of rich non-linear embeddings that could capture desirable data similarities. Further, the decision margin is directly related to the support vectors, which form a weighted data subset. The ability to use slack variables within the SVM formulation allows for a natural control of the influence of false negatives on  the representation learning setup.

A straightforward use of the MMCL objective could be practically challenging. This is because SVMs involve solving a constrained quadratic optimization problem, solving which exactly could dramatically increase the training time when used within standard deep learning models. To this end, inspired by coordinate descent algorithms, we propose a novel reformulation of the SVM objective using the assumptions typically used in contrastive learning setups. Specifically, we propose to use a single positive data sample to train the SVM against the negatives -- a situation for which efficient approximate solutions can be obtained for the discriminative hyperplane. Once the hyperplane is obtained, we propose to use it for representation learning. Thus, we formulate an objective that uses this learned hyperplane to maximize the classification margin between the remaining positives and the negatives. To demonstrate the empirical benefits of our approach to unsupervised learning, we replace the logistic classifier from prior  contrastive learning algorithms with the proposed MMCL objectives. We present experiments on standard benchmark datasets; our results reveal that using our max-margin objective leads to {faster convergence} and needs {far fewer negatives} than prior approaches and produces representations that are {better generalizable} to several downstream tasks, including transfer learning for many-shot recognition, few-shot recognition, and surface normal estimation.

Below, we summarize the key contributions of this work:
\begin{itemize}
    \item We propose a novel contrastive learning formulation using SVMs, dubbed \emph{max-margin contrastive learning}.
    \item We present a novel simplification of the SVM objective using the problem setup commonly used in contrastive learning -- this simplification allows deriving efficient approximations for the decision hyperplane.
    \item We explore two approximate solvers for the SVM hyperplane: (i) using projected gradient descent and (ii) closed-form using truncated least squares.
    \item We present experiments on standard computer vision datasets such as ImageNet-1k, ImageNet-100, STL-10, CIFAR-100, and UCF101, demonstrating superior performances against state of the art, while requiring only smaller negative batches. Further, on a wide variety of transfer learning tasks, our pre-trained model shows better generalizability than competing approaches. 
\end{itemize}

\section{Related Works}
While the key ideas in contrastive learning are  classical \cite{becker1992self,gutmann2010noise,hadsell2006dimensionality}, it has  recently become very popular due to its applications in self-supervised learning. Arguably, objectives based on contrastive learning have outperformed several hand-designed pre-text tasks \cite{doersch2015unsupervised,gidaris2018unsupervised,larsson2016learning,noroozi2016unsupervised,zhang2016colorful}.
Apart from visual representation learning, the idea of contrastive learning is quickly proliferating into several other subdomains in machine learning, including video understanding~\cite{NEURIPS2020_3def184a}, graph representation learning \cite{you2020graph,sun2020infograph}, natural language processing \cite{logeswaran2018an}, and learning audio representations \cite{saeed2020contrastive}. 

 In contrastive predictive coding~\cite{oord2018representation}, which is one of the first works to apply contrastive learning for self-supervised learning, the noise-contrastive loss was re-targeted for representation learning via the pre-text task of future prediction in sequences. It is often empirically seen that the quality of the negatives to be contrasted against has a strong influence on the effectiveness of the representation learned. To this end, for visual representation learning tasks, SimCLR \cite{chen2020simple,chen2020big} proposed a framework that uses a bank of augmentations to generate positives and negatives. As the number of negatives play a crucial role in NCE, many approaches also make use of a memory bank \cite{chen2020improved,he2020momentum,misra2020self,zhuang2019local} to enable efficient bookkeeping of the large batches of negatives. Other contrastive learning objectives include: clustering \cite{caron2018deep,caron2020unsupervised,PCL},  predicting the representations of augmented views \cite{grill2020bootstrap}, and learning invariances~\cite{tian2019contrastive,Xiao2020WhatSN}. The lack of access to class labels in contrastive learning can lead to incorrect learning; e.g., due to false negatives. Recent works have attempted to tackle this issue via avoiding \emph{sampling bias} \cite{chuang2020debiased} and adjusting the contrastive loss for the impact of false negatives \cite{robinson2021contrastive,huynh2020boosting,kalantidis2020hard,Iscen2018MiningOM}. In comparison to these methods that make adjustments to the NCE loss, we propose an alternative way to view contrastive learning through the lens of max-margin methods using support vector machines; allowing for an amalgamation of the rich literature of SVMs with modern deep unsupervised representation learning approaches.

A key idea in our setup is to view the support vectors as hard negatives for contrastive learning via maximizing the decision margin. Conceptually, this idea is reminiscent of hard-negative mining used in classical supervised learning setups, such as deformable parts models~\cite{felzenszwalb2009object}, triplet-based losses~\cite{schroff2015facenet}, and stochastic negative mining approaches~\cite{reddi2019stochastic}. However, different from these methods, we explore self-supervised losses in this paper, which require novel reformulations of max-margin objectives for making the setup computationally tractable. Our proposed approximations to MMCL result in a one-point-against-all SVM classifier, which is similar to exemplar-SVMs~\cite{malisiewicz2011ensemble}; however rather than learning a bank of classifiers for specific tasks, our objective is to learn embeddings that are generalizable and useful for other tasks.

\section{Preliminaries}
In this section, we review our notation and visit the principles of contrastive learning, support vector machines, and their potential connections, that will set the stage for presenting our approach. We use lower-case for single entities (such as $x$), and upper-case (e.g., $\mX$) for matrices (synonymous with a collection of entities). We use lower-case bold-font (e.g., $\vz$) for vectors. For a function, say $f$, defined on vectors, we sometimes overload it as $f(X)$, by which we mean applying $f$ to each entity in $X$. 

\subsection{Contrastive Learning}
Suppose $\dataset=\set{\vx_i}_{i=1}^N$ is a given unlabeled dataset, where each $\vx_i\in\reals{d}$. Let $\trans\!:\reals{d}\to\reals{d}$ denote a random cascade of data transformation maps (e.g., random image crops and rotations). 
Standard contrastive learning methods use $\trans$ to augment $\dataset$, thereby producing sets of data points $\dataset'=\set{\sX_1,\sX_2,\cdots, \sX_N}$, where each $\sX$ is a (potentially infinite) set of transformed data samples obtained via randomly applying $\trans$ on each $\vx$, i.e., $\sX=\set{\trans(\vx)}$. The task of representation learning then amounts to minimizing an objective that maximizes the similarity between points from within a set against data points from other sets -- essentially learning the data manifold in some representation space, with the hope that such representations are useful in subsequent tasks. 

Suppose $\ft\!:\reals{d}\to\reals{d'}$ denote a function mapping a data point $\vx$ to its representation, i.e., $\ft(\vx)$. Then, inspired by noise-contrastive estimation~\cite{gutmann2010noise}, contrastive learning methods learn the function $\ft$ via minimizing the empirical logistic loss (with respect to $\theta$):
\begin{equation}
 -\!\!\sum\limits_{\sX\in\batch}\!\!\log\frac{g(\ft(\vx), \ft(\px))}{g(\ft(\vx),\ft(\px))+ \sum_{\nx\in\batch\backslash\sX}g(\ft(\vx),\ft(\nx))},
    \label{eq:1}
\end{equation}
%over all batches $\batch\subset\dataset'$
 over batches $\batch\subset\dataset'$ with positives $\set{\vx,\px}\subset\sX\in\batch$, negatives $\nx\in\sX'$, where $\sX'\in\batch\backslash\sX$, and using a suitable similarity function $g$ (e.g., a learnable projection-head followed by an exponentiated-cosine distance as in SimCLR~\cite{chen2020simple}). 
As alluded to earlier, the contrastive learning loss in~\eqref{eq:1} poses several challenges from a representation learning perspective. For example, in the absence of any form of supervision, this learning objective needs to derive the training signals from the (thus far learned) representations of the negative pairs, which could be very noisy; thereby requiring very large negative batches. However, having such large batches increases the chances of \emph{class collisions}, i.e., positives and negatives belonging to the same class in a subsequent downstream task; such collisions have been shown to be detrimental~\cite{arora2019theoretical}. As alluded to earlier, unlike approaches that attempt to circumvent this issue, such as~\cite{huynh2020boosting,robinson2021contrastive,chuang2020debiased}, we seek to explore alternative contrastive learning objectives that are less sensitive to issues discussed above using formulations that maximize the discriminative margin between the positives and the negatives. 

Note that instead of the InfoNCE loss, as in ~\eqref{eq:1}, for contrasting the positives from the negatives, an alternative is perhaps the hinge loss~\cite{arora2019theoretical,chen2020simple}, that minimizes (with respect to $\theta$):

\begin{equation}
     \sum_{\vx, \px, \nx} \hinge{\ t - \dist(\ft(\vx), \ft(\px)) + \dist(\ft(x), \ft(\nx))},\notag
    % \label{eq:hinge}
\end{equation}
where $\hinge{\ .\ }=\max(0,\ .)$ denotes the hinge loss and $t$ is a margin hyperparameter that must be tuned manually. Our proposed scheme avoids the need for this hyperparameter as the margin is an objective of the optimization.

\begin{figure}
    \centering
    \includegraphics[width=8.5cm,trim={0.5cm 0cm 0cm 0cm},clip]{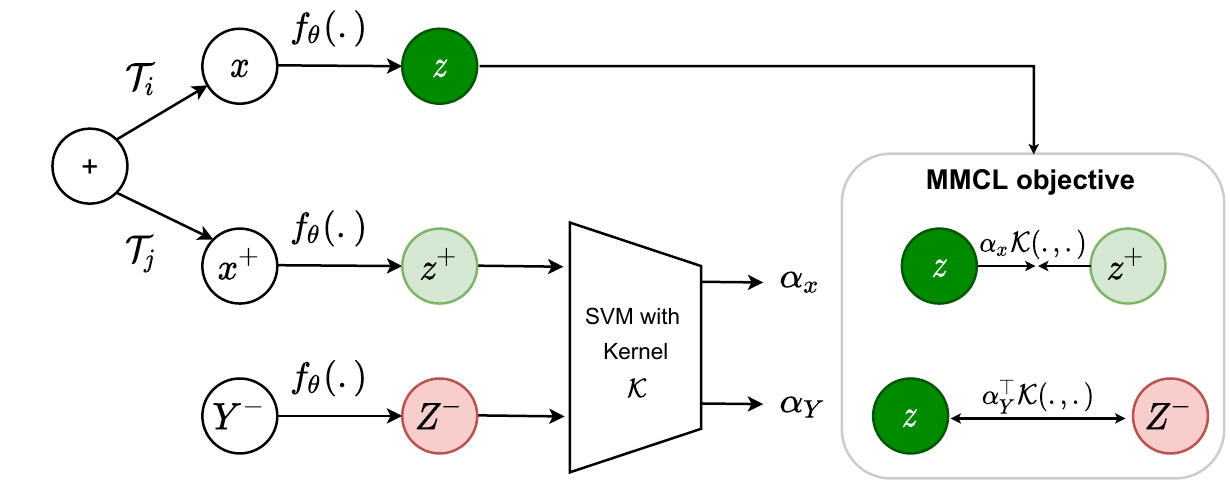}
    \caption{An illustration of our MMCL approach. Given a positive point ($+$) and a set of negatives $\mY^{-}$, MMCL learns the parameters $\theta$ of a backbone network $f_\theta$ via extracting features $z^+$ and $Z^-$ using a view $x^+$ of the positive $+$ and the negatives $Y^-$, respectively. These features are then used in an SVM with an RKHS kernel $\kernel$ to find a decision hyperplane parameterized by $\alpha_x$ and $\valpha_Y$. Next, MMCL uses the remaining positive views $z$ maximizing the similarity between $z$ and $z^+$, while minimizing the similarity between $z$ and $Z^-$, thereby achieving contrastiveness. This ensuing MMCL loss is then backpropagated through the pipeline, thereby learning $\theta$, which is the goal.}
    \label{fig:representative_figure}
\end{figure}

\subsection{Support Vector Machines}
Given two sets $\pX$ and $\nX$ with labels $y_\vx=1$, if $\vx\in\pX$ and $-1$ otherwise, the soft-margin SVM solves the objective:
\begin{align}
    &\min_{\vw, b, \xi\geq 0} \frac{1}{2}\enorm{\vw}^2 + C\sum_{\vx}\xi_\vx \notag\\
    \text{ s. t. }  y_\vx&(\vw^\top\vx + b) \geq 1-\xi_\vx, \forall\vx\in\pX \cup \nX,
    %\text{ and } &\vw^\top\vx \leq -1+\xi_\vx,  \forall\vx\in\nX,
    \label{eq:psvm}
\end{align}
where $\vw$ denotes the discriminative hyperplane separating the two classes, $b$ is a bias, and $\xi_\vx$ is a per-data-point non-negative slack with a penalty $C$ that balances between misclassification of \emph{hard} points and maximizing the decision margin. It is well-known that $1/\enorm{\vw}$ captures the margin between the positives and the negatives, and thus the objective in~\eqref{eq:psvm} attempts to find the hyperplane $\vw$ that maximizes this margin. The Lagrangian dual of~\eqref{eq:psvm} is given by:
\begin{align}
     \min_{0\leq \valpha\leq C, \valpha^\top\vy=0} \half\valpha^{\top}\Kernel(\pX,\nX)\valpha - \valpha^{\top}\vone,
    \label{eq:dsvm}
\end{align}
where $\Kernel\in\spd{|\pX\cup\nX|}$ denotes a symmetric positive definite kernel matrix, the $ij$-th element of which is given by: $\Kernel_{ij}= y_{\vx_i}y_{\vx_j}\kernel(\vx_i, \vx_j)$ for some suitable RKHS kernel $\kernel$ and $\vx_i,\vx_j\in\pX\cup\nX$. As the formulations in~\eqref{eq:psvm} and~\eqref{eq:dsvm} are convex, a solution $\valpha$ to~\eqref{eq:dsvm} provides the exact decision hyperplane for~\eqref{eq:psvm} and is given by:
\begin{equation}
    \vw(.) = \sum_{\vx\in\pX\cup\nX}\!\!\!\!\!\! \alpha_{\vx}y_{\vx}\kernel(\vx, .).
    \label{eq:hyperplane}
\end{equation}
As the bias term $b$ in~\eqref{eq:psvm} is not essential for the details to follow, we will not need the exact form of this term and will use $\vw(.)$ to refer to the decision hyperplane.

\section{Proposed Method}
In this section, we connect the approaches described above deriving our MMCL formulation. An overview of our approach is illustrated in Figure~\ref{fig:representative_figure}.

\subsection{Contrastive Learning Meets SVMs}
The advantages of SVM listed in the last section may seem worthwhile from a contrastive representation learning perspective, and suggest directly using SVM instead of the logistic classifier in~\eqref{eq:1}. Formally, using a soft-constraint variant of~\eqref{eq:psvm} with a margin $t$, the optimization problem in~\eqref{eq:1} can be re-written as:
\begin{align}
  \label{eq:2}
  &\min\limits_\theta\sum\limits_{B \subset D'}\sum\limits_{X \in B} \min\limits_{\vw_X}\Bigl(\tfrac12\enorm{\vw_X}^2 +
  \hinge{t - \ip{\vw_X}{\ft(X)}} + \notag\\
  &\qquad\qquad\qquad\sum\limits_{X^- \in B\backslash X} \hinge{t + \ip{\vw_X}{\ft(X^-)}}\Bigr),
\end{align}
where $X, X^-$ denote the sets of positives and negatives respectively, and $w_X$ captures a max-margin hyperplane separating them.\footnote{Note that $\ft(\Lambda)$ we mean applying $\ft$ to each item in set $\Lambda$.} The inner optimization over \emph{each} $w_X$ is what translates into training an SVM. We augment this inner optimization problem in two ways: (i) by including slack variables to model a soft-margin (as in~\eqref{eq:psvm}), which results in a hyperparameter $C$; and (ii) by permitting an additional nonlinear feature map $\phi$ so that we may use $\phi(\ft)$ (as in~\eqref{eq:dsvm}) in~\eqref{eq:2}. Using these changes, a contrastive learning formulation via maximizing the SVM classification margin may be derived (by rewriting~\eqref{eq:2}) as:

\begin{align}
    &\min_{\theta}\ {\cal L}(\theta) := \sum_{\batch\subset\dataset'}\sum_{\sX\in\batch}  \valpha_X{^*}^{\top}\!\Kernel\left(\ft(\sX), \ft(\batch\backslash\sX)\right)\valpha_X^*,\label{eq:5}\\
        &\text{s.t.}\quad\label{eq:6}\valpha_X^* = \argmin\limits_{\substack{0\leq \valpha\leq C,\\\valpha^\top\vy=0}}\ \tfrac12\valpha^{\top}\Kernel(\ft(\sX), \ft(\batch\backslash\sX))\valpha-\valpha^{\top}\vone,
\end{align}
where $\Kernel(\pZ, \nZ)=\left[\begin{array}{cc}\kernel(\pZ,\pZ),  -\kernel(\pZ, \nZ)\\ -\kernel(\nZ,\pZ), \kernel(\nZ, \nZ)\end{array}\right]$
is a kernel matrix induced by the RKHS kernel $\kernel(\vz,\vz')=\ip{\phi(\vz)}{\phi(\vz')}$. While SVMs have been widely studied in the machine learning literature~\cite{smola1998learning,cortes1995support}, our idea of linking the fields of SVMs and Contrastive Learning has not been explored before.

In~\eqref{eq:6}, we use the so-far trained $\ft$ to produce $\valpha_X^*$ defining the decision margin, which is then used in~\eqref{eq:5} to update $\theta$ while striving to maximize the margin; doing so, pushes the support vectors from the positive and negative classes away from each other. Unfortunately, despite its intuitive simplicity, the formulation \eqref{eq:5}-\eqref{eq:6} is impractical to use directly. Indeed, it is a challenging bilevel optimization problem~\cite{gould2016differentiating,amos2017optnet,wang2018video}, and if we use an iterative SVM solver for the lower problem~\eqref{eq:6} within a deep learning framework, it can incur significant slowdown. 

\textbf{Remarks.} There are several interesting aspects of the SVM solution that are perhaps beneficial from a contrastive learning perspective: (i) the dual solution $\valpha$ is usually sparse\footnote{When the $\kernel$ is chosen appropriately.}, and its active dimensions can be used to identify data points that are the support vectors defining the decision margin, (ii) the slack regularization controls the misclassification rate, and allows tuning the performance against the class collisions, similar to~\cite{chuang2020debiased}, (iii) the dimensions of $\valpha_{X}$ are equal to $C$ for misclassified points, which are perhaps hard or false negatives, and thus our formulation allows for identifying these points and mitigate their effects, and (iv) the use of the kernel function provides rich RKHS similarities at our disposal allowing to use, for example, novel structures within the learned representations (e.g., trees, graphs, etc.). 

\subsection{Max-Margin Contrastive Learning}
\label{sec:mmcl}
The primary method for solving~\eqref{eq:5} is stochastic gradient descent (SGD), which computes stochastic gradients over the batches $B \subset \dataset'$ via backpropagation while iteratively updating $\theta$. However, as has been previously observed for bilevel optimization~\citep{amos2017optnet,gould2016differentiating}, even obtaining a single stochastic gradient requires solving the lower problem~\eqref{eq:6} exactly, which is impractical. Our key idea to overcome this challenge is to introduce a ``\emph{sample splitting}'' trick inspired by coordinate descent, which helps to reduce the computational burden. Subsequently, we make additional approximations that lead to our final training procedure.

Without loss of generality, assume that $X$ consists of the pair $(\vx,\px)$; the same idea applies if we permit multiple such positive pairs in $X$. Instead of solving~\eqref{eq:6} using all the ``coordinates'', we \emph{split} the pair $(\vx,\px)$ into two parts: (i) $\px$, which is used to perform coordinate descent on~\eqref{eq:6}; and (ii) $\vx$, which is used to perform the SGD step for~\eqref{eq:5}. This splitting aligns well with contrastive learning, where often one uses only a pair of positives that must be contrasted against the negatives. 

The following proposition states how we perform part (i) of our split to estimate $\valpha_X^*$, which we will henceforth denote as $\valpha_{\vx}$ to indicate its dependence on the split sample. 

\begin{prop}
  \label{prop:npt}
  Let $(\px,\nY)$ be a tuple consisting of a positive point $\px\in\reals{d}$ and a set of $n$ negative points $\nY\in\reals{d\times n}$. Further, let $\pz=\ft(\px)$ and $\nmZ=\ft(\nY)$. Suppose $\kxx,\kxy,$ and $\Kyy$ denote $\kernel(\pz,\pz)$, $\kernel(\pz, \nmZ)$, and $\kernel(\nmZ, \nmZ)$, respectively. Consider the SVM decision function for a new point $\vz$ given by 
  \begin{equation}
    \label{eq:7}
    \vw(\vz) = \valpha_x^{\top}\left(\kernel(\pz, \vz)\one - \kernel(\nmZ, \vz)\right).
  \end{equation}
  Let $\mDel=\one\tone + \Kyy - \kxy\tone - \one\kxy^{\top}$, and let $P_{[0,C]}$ denote projection onto the interval $[0,C]$. By suitably selecting $\valpha_x$ in~\eqref{eq:7} we then obtain the following approximate max-margin solutions:
  \begin{enumerate}[(i)]
    \setlength{\itemsep}{0pt}
  \item (block) coordinate minimization  $\valpha_{\vx}^{\text{cm}}=\argmin_{0\le \valpha \le C} g(\valpha) := \half\valpha^{\top}\mDel\valpha -2\valpha^{\top}\one$,
  \item m-step projected gradient ($\PGDMMCL$): $\valpha_{\vx}^{\text{pg}} := \valpha_{m}= P_{[0,C]}(\valpha_{m-1}-\eta(\mDel\valpha_{m-1}-2\one))$, for some initial guess $\valpha_0\in [0,C]^n$, $\eta>0$ a step-size, and 
  \item greedy truncated least-squares ($\INVMMCL$): $\valpha_{\vx}^{\text{ls}}=P_{[0,C]}(2\inv{\mDel}\one)$.
  \end{enumerate}
  The various solutions satisfy $g(\mDel^{-1}1) \le g(\valpha_{\vx}^{\text{cm}}) \le \min\{g(\valpha_{\vx}^{\text{pg}}, g(\valpha_{\vx}^{\text{ls}})\}$. Moreover, $g(\valpha_{\vx}^{\text{pg}})-g(\valpha_{\vx}^{\text{cm}})=O\left(\exp\left(-m\frac{\lambda_{\min(\mDel)}}{\lambda_{\max(\mDel)}}\right)(g(\valpha_0)-g(\valpha_{\vx}^{\text{cm}}))\right)$.
\end{prop}
\begin{proof}
Choice (i) is obvious. To obtain (ii) and (iii), consider the following dual SVM formulation:
\begin{equation*}
\min_{0\leq\valpha_{\mY}\leq C} \frac12 \left[\!\!\begin{array}{c}\alpha_x\\\valpha_{\mY}\end{array}\!\!\right]^{\top} \left[\!\!\begin{array}{cc}\kxx &-\kxy^{\top}\\-\kxy&\Kyy\end{array}\!\!\right]\left[\!\!\begin{array}{c}\alpha_x\\\valpha_{\mY}\end{array}\!\!\right] - \left[\alpha_x + \valpha_{\mY}^{\top}\one\right],
\end{equation*}

\noindent where $\alpha_x=\valpha_{\mY}^\top\one$. Substituting for $\alpha_x$, we obtain\footnote{Note that $\alpha_x$ is the scalar Lagrangian dual associated with the data point $\vz$ while $\valpha_{x}$ is the vector of all dual variables associated with the batch $\batch$ when considering $x$ as the positive.}:
\begin{align*}
\min_{0\leq\valpha_{\mY}\leq C} g(\valpha_{\mY}) =  \half\valpha_{\mY}^{\top}\mDel\valpha_{\mY} - 2\valpha_{\mY}^{\top}\one.
\end{align*}
Setting $\nabla g(\valpha_{\mY})=0$, we obtain the unconstrained least-squares solution $2\mDel^{-1}\one$, which we can greedily truncate to lie in the interval $[0,C]$ to obtain (iii). Solution (ii) runs $m$ iterations of projected gradient descent, and hence it also satisfies a linear convergence rate, which rapidly brings it within the optimal solution $\valpha_{\vx}^{\text{cm}}$ at the well-known rate depending on the condition number $\lambda_{\max(\mDel)}/\lambda_{\min(\mDel)}$.
\end{proof}
Using Prop.~\ref{prop:npt}, we can reformulate the contrastive learning objective in ~\eqref{eq:5} as maximizing the margin in classifying the other part of the split, namely the positive point $\vx$, correctly against the negatives. Here, we introduce an additional simplification by rewriting the margin in terms of the separation between $\vx$ and $\nY$, using the decision hyperplane~\eqref{eq:7}. Let $\valpha_x$ denote the solution obtained from Proposition~\ref{prop:npt} using the positive point $\vx$. Then, we rewrite~\eqref{eq:5} into our proposed \emph{max-margin contrastive learning} objective as:
\begin{equation} 
\min_{\theta}\hspace{-7pt}\sum_{\substack{(\vx,\px)\sim\batch\in\dataset'\\\nY=\batch\backslash(\vx,\px)}}\hspace{-18pt}\valpha_{x}^T \big[\kernel\left(\ft(\nY\!),\ft(\vx)\right) - 
    \one\kernel\left(\ft(\px), \ft(\vx)\right)\big].
    \label{eq:mmcl}
\end{equation}

When optimizing for $\theta$, \eqref{eq:mmcl} seeks a representation map $\ft$ that improves the similarity between the positives $(\vx,\px)$ and the dissimilarity between $\vx$ and all the points in $\nY$, achieving a similar effect as in standard contrastive learning objective in~\eqref{eq:1}, but with the advantage of choosing kernels, selecting the support vectors that matter to the decision margin, as well as finding points that are perhaps hard negatives (those at the upper-bound of the box-constraints), all in one formulation. Note that, using the exact solver (i) in Prop.~\ref{prop:npt} turned out to be prohibitively expensive in standard contrastive learning pipelines and thus we do not use that variant in our experiments. In Algorithm~\ref{alg:mmcl}, we provide a pseudocode highlighting the key steps in our approach.

\begin{algorithm}[!t]
\caption{\label{alg:mmcl} Pseudocode for MMCL}
\begin{algorithmic}
    \STATE \textbf{Input:} Dataset $\dataset$, batch size $N$, encoder $\ft$, slack- \\penalty $C$, kernel $\mathcal{K}$, augmentation map $\mathcal{T}$
    \FOR{minibatch $B = \{\bm x_k\}_{k=1}^N\sim\dataset$}
    \STATE $\texttt{loss} := 0$
    \STATE \textbf{for} $k = 1, \ldots, N$ \textbf{do}
        \STATE $~~~~$draw $t_1 \!\sim\! \mathcal{T}$, $t_2 \!\sim\! \mathcal{T}$
        \STATE $~~~~$\textcolor{gray}{\# get embeddings for positives and negatives}
        \STATE $~~~~$$\pz = \ft(t_1(\bm x_k)), \vz = \ft(t_2(\bm x_k))$
        \STATE $~~~~$$\nmZ = \ft(t_1(B \setminus \bm x_k)\cup t_2(B \setminus \bm x_k))$
        \STATE $~~~~$\textcolor{gray}{\# calculate kernel similarities}
        \STATE $~~~~$$\kxy = \kernel(\pz, \nmZ)$, $\Kyy = \kernel(\nmZ, \nmZ)$
        \STATE $~~~~$\textcolor{gray}{\# Solve SVM}
        \STATE $~~~~$$\mDel=\one\tone + \Kyy - \kxy\tone - \one\kxy^{\top}$
        \STATE $~~~~$$\valpha_{\vx} = $ \texttt{svm\_solver}($\mDel$, $C$)~~\textcolor{gray} {\# using PGD or INV}
        \STATE $~~~~$\textcolor{gray}{\# calculate the loss}
        \STATE $~~~~$\texttt{loss} += $ \valpha_x^{\top}\left(\kernel(\nmZ, \vz) - \kernel(\pz, \vz)\one\right)$
        
    \STATE \textbf{end for}
    \STATE update the model to minimize the loss
    \ENDFOR
\end{algorithmic}
\end{algorithm}

\section{Experiments and Results}
\begin{table*}[h]
\resizebox{\linewidth}{!}{
\begin{tabular}{l|cccccccc|cc|c}
\toprule
& \multicolumn{8}{c}{Many-Shot classification} & \multicolumn{2}{c}{Few-Shot classification} & Surface Normal Est. \\
 \textit{Method} & Aircraft & Caltech101 & Cars & CIFAR10 & CIFAR100 & DTD & Flowers & Food &  CropDiseases & EuroSat & NYUv2 (Angular error) \\
 \midrule
 Supervised  & 83.5 & 91.01 & 82.61 & 96.39 & 82.91 & 73.30 & 95.50 & 84.60 & 93.09 $\pm$ 0.43 & 88.36 $\pm$ 0.44 & 27.91 \\
 \midrule
InsDis~\cite{wu2018unsupervised} & 73.38 & 72.04 & 61.56 & 93.32 & 68.26 & 63.99 & 89.51 & 76.78 & 91.95 $\pm$ 0.44 & 86.52 $\pm$ 0.51 & 27.35 \\
MoCo~\cite{he2020momentum} & 75.61 & 74.95 & 65.02 & 93.89 & 71.52 & 65.37 & 89.45 & 77.28 & 92.04 $\pm$ 0.43  & 86.55 $\pm$ 0.51 & 28.63 \\
PCL-v1~\cite{PCL} & 74.97 & \textbf{87.62} & 73.24 & \textbf{96.35} & 79.62 & 70 & 90.83 & 78.3 & 80.74 $\pm$ 0.57 & 75.19 $\pm$ 0.67 & 33.58\\
PCL-v2~\cite{PCL} & 79.37 & 88.04 & 71.68 & \textbf{96.5} & 80.26 & 71.76 & 92.95 & 80.34 &92.58 $\pm$ 0.44 & 87.94 $\pm$ 0.4 & 28.67 \\
MoCo-v2~\cite{chen2020improved}$^\dagger$  & 82.46 & 82.31 & 85.1 & 96.06 & 72.99 & 69.41 & \textbf{95.62} & 77.19 & 90.01 $\pm$ 0.48 & 88.06 $\pm$ 0.4 & \textbf{24.49} \\
MoCHI~\cite{kalantidis2020hard}$^\dagger$ & 83.03 & 84.45 & 85.49 & 95.68 & 77.07 & 70.85 & 94.8 & 78.9 & 91.93 $\pm$ 0.46 & \textbf{88.26} $\pm$ \textbf{0.4} & 31.75 \\
\midrule
Ours & \textbf{85.38} & \textbf{87.82} & \textbf{89.23} & \textbf{96.24} & \textbf{82.09} & \textbf{73.51} & \textbf{95.24} & \textbf{82.39} & \textbf{93.1} $\pm$ \textbf{0.45} & \textbf{88.75} $\pm$ \textbf{0.4} & \textbf{24.69}\\
\bottomrule
\end{tabular}}
\caption{Transfer learning results. We transfer an ImageNet-pre-trained model (using MMCL) on a range of downstream tasks and datasets. We compare with models pre-trained using a similar batch size and epochs. Results on competing approaches are taken from~\cite{Ericsson2021HowTransfer}. $^\dagger$Models evaluated using publicly available checkpoints.}
\label{tab: transfer_multishot}
\end{table*}

In this section, we systematically study the various components in MMCL, as well as compare performances of MMCL-learned representations for their quality via linear evaluation, and their generalizability on transfer learning tasks.

\noindent \textbf{Visual Representation Learning Experiments.} We base our experimental setup on the popular SimCLR~\cite{chen2020simple} baseline, which is widely used, especially to evaluate the effectiveness of the ``learning loss'' against other factors in a self-supervised algorithm (e.g., data augmentations, use of queues, multiple crops).
We use a ResNet50 backbone, followed by a two-layer MLP as the projection-head, followed by unit normalization. We pretrain our models on ImageNet-1K~\cite{deng2009imagenet} using the LARS optimizer~\cite{you2017large} with an initial learning rate of 1.2 for 100 epochs. We also present results on ImageNet-100~\cite{tian2019contrastive} (a subset of ImageNet-1K) and on smaller datasets such as STL-10~\cite{coates2011analysis} and CIFAR-100~\cite{krizhevsky2009learning}, especially for our ablation studies. We pre-train on ImageNet-100 for 200 epochs in our studies, while pre-training for 400 epochs on the smaller datasets. We use the Adam optimizer with a learning rate of 1e-3 as in~\cite{chuang2020debiased,robinson2021contrastive}. Unless otherwise stated, we use a batch-size of 256 for all ImageNet-1K, CIFAR-100, and STL-10 experiments and 128 for ImageNet-100 experiments. In addition, we also present results on video representation learning using an S3D backbone~\cite{Xie2018RethinkingSF} on the UCF-101~\cite{soomro2012ucf101} dataset, pre-trained using MMCL for 300 epochs. 
\\
\noindent\textbf{Hyperparameters:} 
We mainly use the RBF kernel for the SVM. For CIFAR-100 experiments, we start with a kernel bandwidth $\sigma^2=0.02$ and increase it by a factor of 10 at 75 and 125 epochs. For STL-10 experiments, we use a kernel bandwidth $\sigma^2=1$. We used $\sigma^2=5$ for ImageNet experiments. 
We set the SVM slack regularization $C$ to 100.  For the projected gradient descent optimizer for MMCL, we use a maximum of 1000 steps. Additional details are provided in the Appendix.

\noindent\textbf{Practical Considerations:}
Here, we note a few important but subtle technicalities that need to be addressed when implementing MMCL. Specifically, we found that backpropagating the gradients through $\valpha_Y$ is prohibitively expensive when using PGD iterations. On the other hand, for the least-squares variant, gradients through $\valpha_Y$ was found to be detrimental. This is perhaps unsurprising, because note that, $\valpha_Y$ term includes the term $\inv{\Delta}$. To improve the decision margin, one needs to make $\Delta$ an identity matrix, so that the off-diagonal elements go to zero during optimization, which suggests that the training gradients should reduce the magnitude of these terms. However, on the other hand, as $\valpha_Y$ uses $\inv{\Delta}$ one could also maximize the margin by making $\Delta$ ill-conditioned, via making the off-diagonal elements going to one. Such a tug-of-war between the gradients can essentially destabilize the training. Thus, we found that avoiding any backpropagation through $\valpha$ is essential for MMCL to learn to produce representations. We also found that using a small regularization $\Delta+\beta I$ ($\beta=0.1$) is necessary for the learning to begin. This is because, initially the representations can be nearly zero, and thus the kernel may be poorly conditioned.

\subsection{Experiments on Transfer Learning}
Recently, models pretrained using various self-supervised learning approaches have shown impressive performance when transferred to various downstream tasks. In this section, we evaluate MMCL-ImageNet pretrained models on such downstream tasks. For these experiments, we follow the experimental protocol provided in \cite{Ericsson2021HowTransfer}. We evaluate the models in the fine-tuning setting and use the benchmarking scripts provided in~\cite{Ericsson2021HowTransfer} without any modifications. 

First, we transfer the MMCL-pretrained backbone model to a collection of many-shot classification datasets used in~\cite{Ericsson2021HowTransfer}, namely FGVC Aircraft, Caltech-101, Stanford Cars, CIFAR-10, CIFAR-100, DTD, Oxford Flowers, and Food-101. The setup involves using the pretrained model as the initial checkpoint and attaching a task specific head to the backbone model. The entire network is then finetuned for the downstream task. These datasets vary widely in content and texture compared to ImageNet images. Further, the benchmark datasets include a significant diversity in the number of training images (2K--50K) and the number of classes (10--196). For a fair comparison, we only include results for models which are trained for a comparable number of epochs and batch sizes. For few-shot experiments, we follow the setup described in~\cite{Ericsson2021HowTransfer} for few-shot learning on the Cross-Domain Few-Shot Learning (CD-FSL) benchmark~\cite{guo2020broader}. We evaluate on Crop-Diseases~\cite{mohanty2016using}, EuroSAT~\cite{helber2019eurosat} datasets for 5-way 20-shot transfer. Finally, we evaluate performance of our model for the dense prediction task of surface normal estimation on NYUv2~\cite{silberman2012indoor} and report the median angular error. 

In Table~\ref{tab: transfer_multishot}, we provide results on the transfer learning experiments. We see that MMCL consistently outperforms the competing self-supervised learning approaches on a wide variety of transfer tasks and across all datasets. Further, MMCL also outperforms the supervised counterpart on several datasets. These results show that MMCL learns high-quality generalizable features.
\subsection{Experiments on Linear Evaluation}
For these experiments, we freeze the weights of the backbone (ResNet-50), and attach a linear layer as in~\cite{chen2020simple}, which is trained using the class labels available with the dataset. We train this linear layer for 100 epochs. Tables~\ref{tab: imagenet_100} and \ref{tab: imagenet_1k} show our results. We see that MMCL-pretrained model outperforms SimCLR by 6.3\% on ImageNet-1K using the same number of negatives. We also compare with the recent memory queue-based methods such as MoCo-v2~\cite{chen2020improved} and MoCHI~\cite{kalantidis2020hard}, demonstrating competitive performances while using far fewer negatives (510 vs 65536). We also establish a new state of the art on ImageNet-100 outperforming MoCHI by 1.7\% using only 510 negatives (0.008x) and without a memory bank.
\begin{table}[!htb]
    \centering
    \begin{tabular}{lcccr}
        \toprule
        Variant & Negative Source & Negatives & top-1 \\
        \toprule
        MoCo & Memory Queue & 16000 & 75.9 \\
        CMC & Memory Queue &16000 & 75.7 \\
        MoCo-v2  & Memory Queue & 16000 & 78.0 \\
        MoCHI & Memory Queue & 16000 & 79.0 \\
        \midrule
        Ours & Batch & 254 & \textbf{80.7} \\
        \bottomrule
        \end{tabular}
        \caption{ImageNet-100 linear evaluation.
        \label{tab: imagenet_100}
        }
\end{table}

\begin{table}[!htb]
    \centering
    \begin{tabular}{lcccr}
        \toprule
        Variant & Negative Source & Negatives & top-1 \\
        \toprule
        SimCLR  & Batch & 254 & 57.5 \\
        SimCLR  & Batch & 510 & 60.62 \\
        \midrule
        MoCo-v2 & Memory Queue & 65536 & 63.6 \\
        MoCHI & Memory Queue & 65536 & \textbf{63.9} \\
        \midrule
        Ours & Batch & 510 & \textbf{63.8} \\
        \bottomrule
        \end{tabular}
        \caption{ImageNet-1K linear evaluation. 
        \label{tab: imagenet_1k}}
\end{table}

\begin{table*}[h]
\label{sample-table}
\vskip 0.15in
\begin{center}
\begin{tabular}{lccccr}
\toprule
Method  & DD & MUTAG  & REDDIT-BIN & REDDIT-M5K & IMDB-BIN \\
\midrule
GraphCL & $\mathbf{78.62 \pm 0.40}$ & $86.80 \pm 1.34$ & $89.53 \pm 0.84$ & $55.99 \pm 0.28$ & $71.14 \pm 0.44$ \\
\midrule
Ours & $\mathbf{78.74 \pm 0.30}$ & $\mathbf{88.42 \pm 1.33}$ & $\mathbf{90.41 \pm 0.60}$ &  $\mathbf{56.18 \pm 0.29}$ & $\mathbf{71.62 \pm 0.28}$ \\
\bottomrule
\end{tabular}
\end{center}
\caption{Comparison with GraphCL. We compare graph representation learning on five graph benchmark datasets. The compared numbers are obtained from the original paper~\cite{you2020graph}.}
\label{tab: graph}
\vskip -0.1in
\end{table*}

\begin{figure*}[!h]
    \centering
    \begin{minipage}{0.3\textwidth}
        \centering
        \includegraphics[width=0.95\textwidth]{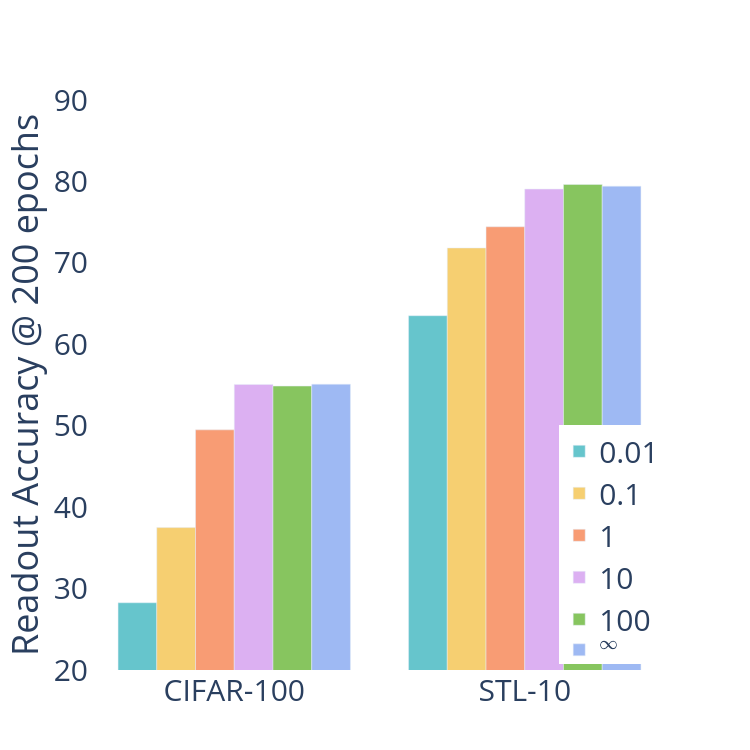} % first figure itself
        \caption{Effect of slack penalty $C$.}
        \label{fig: main_sigma_ablation}
    \end{minipage}
    \begin{minipage}{0.3\textwidth}
        \centering
        \includegraphics[width=0.95\textwidth]{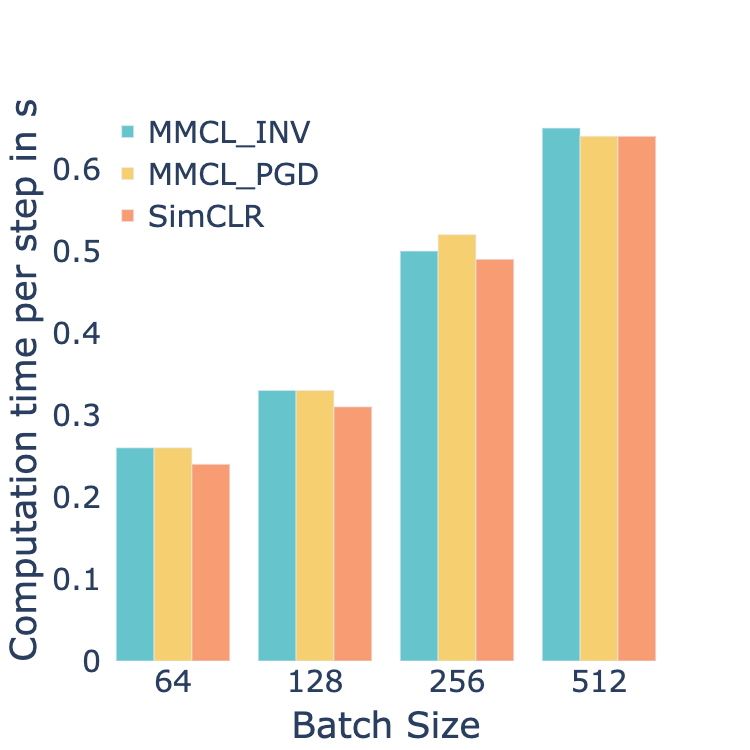} % second figure itself
        \caption{Computations (ImageNet).}
        \label{fig: compute_time}
    \end{minipage}  
    \begin{minipage}{0.3\textwidth}
        \centering
        \includegraphics[width=0.95\textwidth]{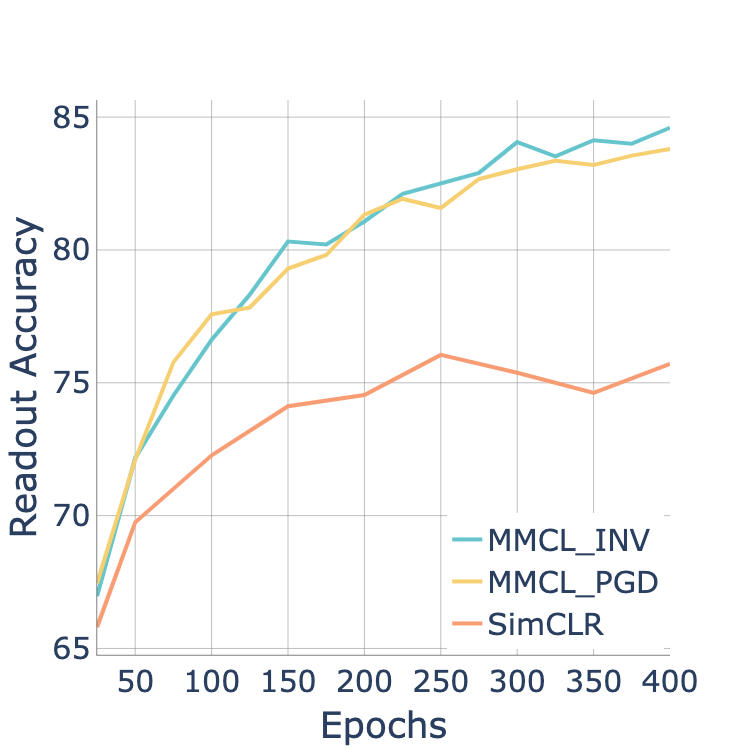} 
    \caption{Convergence (STL-10).}
    \label{fig:convergence}
    \end{minipage}
\end{figure*}

\subsection{Experiments on Graph Representation Learning}
Recall that our MMCL formulation works by modifying the contrastive learning loss function; and as a result, our approach is generically applicable to a variety of tasks. In this section, we evaluate our approach on learning graph representations using contrastive learning. We experiment with five common graph benchmark datasets MUTAG \cite{Kriege2012SubgraphMK} -- a dataset containing mutagenic compounds, DD\cite{yanardag2015deep}-- a dataset of biochemical molecules, REDDIT-BINARY, REDDIT-M5K, and IMDB-BINARY \cite{yanardag2015deep} which are social network datasets. Our experiments use GraphCL~\cite{you2020graph} --  a projection head based on the contrastive learning framework derived from SimCLR while incorporating graph augmentations.
For these experiments, we follow the training and evaluation protocols described in \cite{you2020graph}. Specifically, we use the standard ten-fold cross validation using an SVM and report the average performances and their standard deviations. We use the Adam optimizer for training these models. Table~\ref{tab: graph} shows the results of using MMCL instead of the NCE loss. We see that adding MMCL is comparable or better than GraphCL for these datasets. On MUTAG, we obtain an absolute improvement of 1.62\% over GraphCL. These results demonstrate the effectiveness of our approach in learning better representations. Given that the only change from GraphCL is the underlying objective, the results also show that our approach is general and can easily replace NCE based losses.
\begin{table}[!htb]
    \centering
    \begin{tabular}{lccccr}
        \toprule
        Loss Variant & Modality & Negatives & top-1 & R@1\\
        \toprule
        MoCo-v2 & RGB & 2048 & 46.8 & 33.1\\
        Ours & RGB & 254 & \textbf{52.45} & \textbf{45.6} \\
        \midrule
        MoCo-v2 & Flow & 2048 & 66.8 & 45.2\\
        Ours& Flow & 254 & \textbf{68.01} & \textbf{50.94} \\
        \bottomrule
        \end{tabular}
        \caption{Video self-supervised learning on UCF-101 dataset.}
    \label{tab: videossl}
\end{table}
\subsection{Experiments on Video Action Recognition}
For this experiment, we use the S3D backbone model~\cite{Xie2018RethinkingSF} pre-trained using MMCL on RGB and optical flow images from the UCF-101 dataset. We pre-train the network for 300 epochs, followed by 100 epochs for linear evaluation on the task of action recognition. We report the standard 10-crop test accuracy on split-1, as well as on nearest neighbor retrieval. As seen in Table~\ref{tab: videossl}, MMCL outperforms the baseline by 5.65\% on RGB and 1.21\% on flow in linear evaluation and 12.5\% and 5.74\% on Retrieval@1, demonstrating the generalizability of our approach to the video domain. 
\subsection{Ablation studies and Analyses}
For some of the ablation experiments, we use the smaller datasets: STL-10 and CIFAR-100, and report the readout accuracy calculated using k-NN with k=200 at 200 epochs, besides standard evaluations. 
\\
\noindent\textbf{Choice of Kernels:} Unlike the traditional NCE objective, our approach naturally allows for the use of kernels to better capture the similarity between the data points. In Table~\ref{tab: kernels}, we compare the readout accuracy on CIFAR100 and STL10 for various choices of popular kernels.  As is clear from the table, the RBF kernel performs better on both datasets. The best kernel hyperparameters $\sigma,\gamma$ were found empirically. We choose the RBF kernel in our subsequent experiments.

\begin{table}[]
    \centering
    \begin{tabular}{lrc}
        \toprule
        Kernel ($K(x,y)$) & CIFAR100 & STL10\\
        \midrule
        Linear: $x^Ty$ & 41.43\% & 74.82\%\\
        Tanh: $\tanh(-\gamma x^Ty+\eta)$ & 54.53\% & 80.5\% \\
        RBF: $\exp(-\frac{\|x-y\|^2}{2\sigma^2})$ & \textbf{55.35} \% & \textbf{81.33}\% \\
        \bottomrule
        \end{tabular}
   \caption{Effect of kernel choice.}
    \label{tab: kernels}
\end{table}

\noindent\textbf{Effect of slack:} A key benefit of our MMCL formulation is the possibility to use a slack that could potentially control the impact of false or hard negatives. To evaluate this effect,  we changed the slack penalty $C$ from 0.01 (i.e., low penalty for misclassification) to $C=\infty$. The results on readout accuracy in Figure~\ref{fig: main_sigma_ablation} shows that $C$ plays a key role in achieving good performance. For example, with $C=0.01$, it appears that the performance is consistently low for both the datasets, perhaps because the hard negatives are under-weighted. We also find that using a large $C$ may not be beneficial always.
\\
\noindent\textbf{Effect of batch size:}
We use STL-10 dataset for this experiment and train all models for 400 epochs. We report the Linear evaluation results for this experiment. From Table~\ref{tab: small_datasets_sota}, we see that our model consistently outperforms the SimCLR baseline, while performing much better than other approaches. Indeed, we find that MMCL is about 1-3\% better than HCL, which reweights the hard negatives. We also find that MMCL using a batch size of only 128 reaches close to the performance of HCL~\cite{robinson2021contrastive} trained using a batch size of 256, suggesting that the proposed selection of hard negatives via support vectors is beneficial. 
\\
\begin{table}[]
    \centering
    \begin{tabular}{lccr}
            \toprule
            STL-10  & 64 & 128 & 256  \\
            \midrule
            SimCLR~\cite{chen2020simple}  & 74.21 & 77.18 &  80.15  \\
           DCL~\cite{chuang2020debiased}  & 76.72 & 81.09 & 84.26     \\
          HCL~\cite{robinson2021contrastive}  & 80.39 & 83.98 & 87.44   \\
            \midrule
            Ours  & 80.11 & \textbf{86.8} & \textbf{88.3} \\
            \bottomrule
            \end{tabular}
    \caption{Accuracy (in \%) against the batch size on STL-10.}
    \label{tab: small_datasets_sota}
\end{table}
\noindent\textbf{Computational time against batch size:}
In Figure~\ref{fig: compute_time}, we show the time taken per iteration of MMCL variants against those of prior methods, such as SimCLR, for an increasing batch size. The computational cost of our inner optimization for finding the support vectors is directly related to the batch size. These experiments are done on ImageNet-1K with each RTX3090 GPU holding 64 images. We see that the time taken by MMCL is comparable to SimCLR.
\\
\noindent\textbf{Performances between MMCL Variants:}
In Table~\ref{tab: mmcl_variants_compare}, we compare performances between MMCL variants: PGD and INV. We see that both variants outperform the SimCLR, while the two MMCL variants show similar performance. In Figure~\ref{fig:convergence}, we plot the convergence curves (readout accuracy) against training epochs. The plots clearly show that our variants converge to superior performances rapidly than the baseline.
\\
\noindent\textbf{Visualization of Support Vectors:}
Next, we qualitatively analyze if the support vectors found by MMCL are semantically meaningful. To this end, we use an MMCL model pre-trained on STL-10 dataset. We use a batch of examples as input to the model, and choose one of the examples from the batch as a positive and the remaining as negative. We then solve the MMCL objective to find $\alpha$, where $\alpha=0$ corresponds to non-support vectors, $\alpha=C$ are the misclassified points, and $\alpha\in[0,C)$ are the support vectors. In Figure~\ref{fig : sv_vis}, we show the positive point, and a set of samples from the batch and the respective $\alpha$ values. The figure clearly shows that object instances from a similar class gets a high $\alpha$,
%suggesting that they are closer to the positive point in the latent space,
suggesting that they lie on or inside the margin and contribute to the loss 
while batch samples that are irrelevant or easy negatives
are not support vectors and do not contribute to the loss. For example, in Figure~\ref{fig : sv_vis}, the \emph{yellow bird} is an easy negative for a \emph{white truck} query image and our approach does not include that \emph{bird} in the support set.

\begin{figure*}[ht]
    \centering
    \includegraphics[trim=0 96 0 0,clip,width=0.8\textwidth]{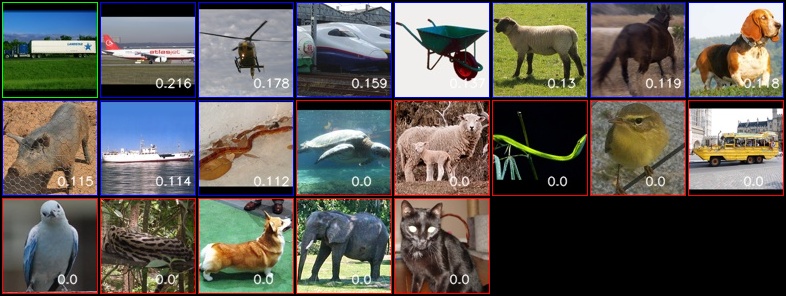}\\% first figure itself=
     \vspace{0.2cm}
    \includegraphics[trim=0 96 0 0,clip,width=0.8\textwidth]{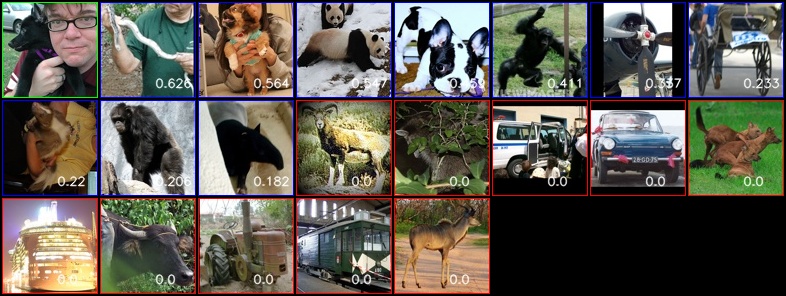}\\ % first figure 
    \vspace{0.2cm}
    \includegraphics[trim=0 96 0 0,clip,width=0.8\textwidth]{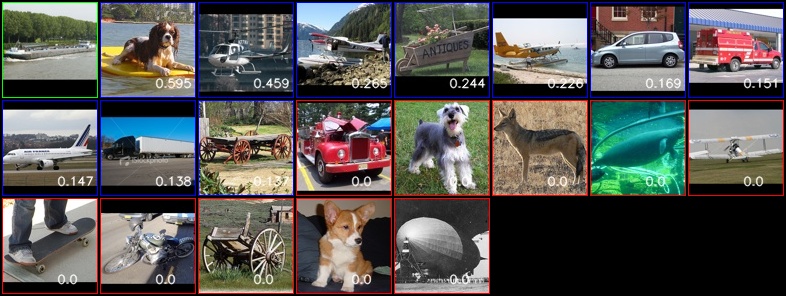}\\ % second figure 
    \vspace{0.2cm}
    \caption{Visualizing Support Vectors : We visualize a query image (green box), corresponding support vectors (blue boxes) and non-support vectors (red boxes). We see that the support vectors are plausible hard negatives while in most cases the non-support vectors are easy negatives. The $\alpha$ corresponding to the various negatives is shown at the bottom left of each image.}
    \label{fig : sv_vis}
\end{figure*}

\textbf{Longer Training:}
 Our focus in the above experiments has been in improving convergence and negative utilization with limited training. However, we see competitive performances on longer training as well. Using a batch size of 256 (510 negatives), our model reaches 66.5\% in 200 and 69.9\% in 400 epochs, compared to 62\% and 64.5\% respectively with same number of negatives in SimCLR\cite{chen2020simple}. Remarkably, our 100 epochs pre-trained models transfer better than PCL-v2’s~\cite{li2020prototypical} 200 epoch models on most transfer learning tasks (see Table~\ref{tab: transfer_multishot}).
\\

\begin{figure}[]
    \centering
    \begin{minipage}{0.33\textwidth}
        \resizebox{\linewidth}{!}{
        \begin{tabular}{lccr}
        \toprule
        Variant  & CIFAR-100 & STL-10 \\
        \midrule
        SimCLR & 66 & 80.15  \\
        \midrule
        $\PGDMMCL$ & \textbf{68.0} & 88.03  \\
            $\INVMMCL$ & \textbf{68.81} & \textbf{88.3} \\
        \bottomrule
        \end{tabular}
        }
        \captionof{table}{Performances on MMCL variants.}
        \label{tab: mmcl_variants_compare}
    \end{minipage}
    \vspace*{-0.5cm}
\end{figure}

\vspace*{-0.5cm}

\section{Conclusions}
In this paper, we proposed a new contrastive learning framework, dubbed Max-Margin Contrastive Learning, using which we learn powerful deep representations for self-supervised learning by maximizing the decision margin separating data pseudo-labeled as positives and negatives. Our approach draws motivations from the classical support vector machines via modeling the selection of useful negatives through support vectors. We obtain consistent improvements over baselines on a variety of downstream tasks.

\section{Acknowledgements}
AS thanks Ketul Shah, Aniket Roy, Shlok Mishra, and Susmija Reddy for their feedback. AS and RC are supported by an ONR MURI grant N00014-20-1-2787. SS acknowledges support from NSF-TRIPODS+X:RES  (1839258) and from NSF-BIGDATA (1741341).

\section{Appendix}
% In this Appendix, we provide additional empirical studies, analyses, insights into our formulation, mathematical proofs and additional details. We summarily list below the key sections of this Appendix. 
% \begin{enumerate}
%     \item Additional Results : transfer learning on 5 additional datasets. 
%     \item tSNE plots {\color{red} Should we remove tSNE plots from the appendix too?}
%     \item Additional analyses and ablation experiments 
%     \item Alternative MMCL variants and other theoretical results
%     \item Implementation Details
%     \item Limitations and potential negative impact
% \end{enumerate}

\section{Additional Transfer Learning Experiments}
In this section, we report results from several additional experiments we conducted analyzing the generalizability of MMCL-learned representations. Specifically, we used SUN397~\cite{xiao2010sun}, Oxford-IIIT Pets~\cite{parkhi2012cats}, and VOC2007~\cite{everingham2010pascal} for many-shot classification, ISIC~\cite{tschandl2018ham10000,codella2019skin} and ChestX~\cite{wang2017chestx} for few-shot classification, and VOC 2007~\cite{everingham2010pascal} for object detection. We follow the standard benchmarking protocol in \cite{Ericsson2021HowTransfer} for evaluating all our results. We report the top-1 accuracy metric on Food,
CIFAR-10, CIFAR-100, SUN397, Stanford Cars, and DTD datasets. The
mAP accuracy is reported for Aircraft,
Pets, Caltech101, and Flowers and the 11-point
mAP metric is reported on Pascal VOC 2007.  For few-shot transfer experiments, we report
the average accuracy along with a 95\% confidence interval. We report AP for object detection on VOC and median angular error for surface normal estimation on NYUv2. Please refer to ~\cite{Ericsson2021HowTransfer} for additional details. For a fair comparison, we only compare to models which were pre-trained for a comparable number of epochs (upto 200) and batch-size of 256. To enable comparison with approaches like MoCo-v2 and MOCHI, we download their official pre-trained models (for a batch size of 256) and follow the same benchmarking process to report the numbers.\footnote{Note that comparing with the large batch size (as high as 4096) models and longer pre-trained models (as high as 1000 epochs) is not fair since they require as many as 16 high-performance GPUs \cite{caron2020unsupervised} and very long pretraining times. For comparison, on our compute setup, training with a 256 batch size model for 100 epochs takes upwards of 3.5 days.} All results are reported in ~\ref{tab: transfer_multishot_supple}. We see that our approach consistently outperforms the prior approaches on most tasks and datasets which show the high-quality nature of representations learned using MMCL. 

\begin{table*}[ht]
\resizebox{\linewidth}{!}{
\begin{tabular}{l|ccc|cc|c}
\toprule
& \multicolumn{3}{c}{Many-Shot classification} & \multicolumn{2}{c}{Few-Shot classification} & \multicolumn{1}{c}{Object Det.}   \\
 \textit{Method} &  Pets & SUN397 & VOC2007 & ISIC & ChestX & VOC (AP) \\
 \midrule
 Supervised   & 92.42 & 63.56 & 84.76  & 48.79 $\pm$ 0.53 & 29.26 $\pm$ 0.44 & 53.26  \\
 \midrule
InsDis~\cite{wu2018unsupervised}  & 76.22 & 51.84 & 71.90  & 52.19 $\pm$ 0.53 & 29.13 $\pm$ 0.44  &  48.82   \\
MoCo~\cite{he2020momentum}  & 76.96 & 53.35 & 74.61 & \textbf{53.79} $\pm$ \textbf{0.54} & \textbf{30.0} $\pm$ \textbf{0.43} &  50.51 \\
PCL-v1~\cite{PCL}  & 86.98 & 58.40 & 82.08 & 38.01 $\pm$ 0.44 & 25.54 $\pm$ 0.43 & \textbf{53.93} \\
PCL-v2~\cite{PCL} & 85.39 & 58.82 & 82.20 & 44.4 $\pm$ 0.52 & 28.28 $\pm$ 0.42 & \textbf{53.92} \\
MoCo-v2~\cite{chen2020improved}$^\dagger$  & 87.32 & 56.22 & 79.34 & 49.70 $\pm$ 0.51 & 29.48 $\pm$ 0.45 & 50.67  \\
MoCHI~\cite{kalantidis2020hard}$^\dagger$ & 87.4 & 57.74 & 79.73 & 49.0 $\pm$ 0.53 & 28.4 $\pm$ 0.4 & 52.61  \\
\midrule
Ours &  \textbf{87.81} & \textbf{62.78} & \textbf{82.83} & 48.79 $\pm$ 0.53 & 29.26 $\pm$ 0.44 & 50.73 \\
\bottomrule
\end{tabular}}
\caption{Transfer learning results. We transfer a ImageNet-pretrained model (using MMCL) on a range of downstream tasks and datasets. We compare with models pretrained using a similar batch size and epochs. Results on competing approaches are taken from~\cite{Ericsson2021HowTransfer}. $^\dagger$Models evaluated using publicly available checkpoints.}
\label{tab: transfer_multishot_supple}
\end{table*}

\section{Additional Ablation Studies}
In this section, we include some additional analyses and experiments to analyse the MMCL loss. 

\noindent\textbf{Effect of Slack Penalty $C$:} In the main paper, we explored the effects of the slack penalty $C$ on performance. In this section, we augment that study with an analysis of the behavior of our model towards a changing $C$. Specifically, note that when starting to train the model, the neural network weights are randomly initialized, and as a result, the generated contrastive learning features might be arbitrary that asking for a misclassification penalty might be unlikely to be useful. To empirically understand this conjencture, we train a model from scratch for 25 epochs using a slack $C=\infty$ (i.e., no misclassification is allowed) and then reduce $C$ to a fixed value for the next 75 epochs. In Table~\ref{tab: reducing_C}, we provide the result of this analysis on CIFAR-100 and STL-10 datasets, where the legend shows the fixed value of $C$ used. The table shows that there is some benefit of using this approach, for example, on STL-10 dataset, where we see that for $C$=10 shows a slight benefit over $C=\infty$ (i.e., not using the proposed approach). However, overall it looks like such a scheme does not  reduce the performance dramatically, unless we underpenalize the misclassifications (such as using C=0.1, for example).

\begin{table}[]
    \centering
    \begin{tabular}{lcr}
    \toprule
      & CIFAR-100 & STL-10 \\
    \midrule
    C=0.1  & 44.41 & 73.82  \\
    C=1  & 53.08 &  78.86  \\
    C=10  & 55.1 &   81.25  \\
    C=50  & 54.95 & 79.63   \\
    C=$\infty$  & 55.1  & 79.42  \\
    \bottomrule
    \end{tabular}
     \caption{Reducing C over epochs.}
    \label{tab: reducing_C}
\end{table}

\noindent\textbf{False Negative Correction:}
In the main paper, we argue that the use of slack allows to limit the effect that false negatives have on training by limiting the maximum contribution that they can make to the loss. An alternative idea would be to nullify all points that lie \emph{inside} the margin (i.e., points that are potentially close to positives, and thus are perhaps false negatives). This would amount to the operation \texttt{$\alpha$[$\alpha$ == C] = 0} after finding the optimal dual values $\alpha$. For this experiment, we use a single $C$ for the entire training (unlike the above experiment). The result of this experiment is shown in Figure~\ref{tab: C_FN_correction}. The table shows that FN correction has a significant impact on performance when using a small $C$, while insignificant for larger $C$. This is perhaps not unexpected given that using a small C will lead to a large number of misclassified points, which if removed from training, will lead to suboptimal contrastive learning, especially in the initial epochs of training. However, with larger values of $C$, it appears from the table that the impact of removing some points does not impact the performance much. This is perhaps because using the same value of $C$ for all the training epochs is perhaps sub-optimal, and would need a schedule for changing this parameter. We plan to explore this idea in the future.

\begin{table}[]
    \centering
    \begin{tabular}{lcr}
    \toprule
      & Readout Acc. \\
    \midrule
    C=1  & 61.12  \\
    C=2 & 69   \\
    C=10  & 80.05  \\
    C=50  & 79.91    \\
    No Correction  & 79.63   \\
    \bottomrule
    \end{tabular}
     \caption{Using false negative correction for STL-10.}
    \label{tab: C_FN_correction}
\end{table}

\noindent\textbf{Influence of RBF Kernel Bandwidth $\sigma$:}
In this experiment we vary the RBF Kernel bandwidth for CIFAR-100 and STL-10. From Figure~\ref{fig: main_gamma_ablation},  we see that the performance can be influenced by the choice of $\sigma$, however, using $\sigma=1$ performs reasonably well for these datasets. 

\begin{figure}
        \centering
        \includegraphics[width=0.7\linewidth]{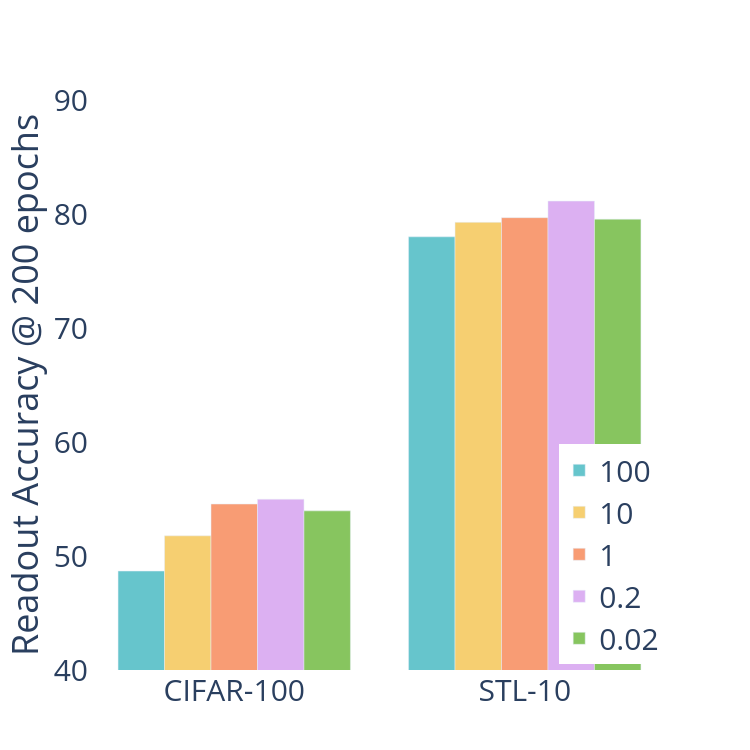} % second figure itself
        \caption{Effect of kernel bandwidth $1/\sigma^2$.}
        \label{fig: main_gamma_ablation}
\end{figure}

\noindent\textbf{Effect of Batch Size:}
In the main paper, we showed the effect of batch size on STL-10 (Table~\ref{tab: small_datasets_sota}). In Table~\ref{tab: batchsize_cifar100}, we augment that with a similar study for CIFAR-100. Our performances on CIFAR-100 are similar, albeit the improvements are not as pronounced as in STL-10, perhaps because its a smaller dataset. Further, the trend in the tables show that similar to other methods, the performance of MMCL increases consistently with increasing batch sizes, while showing signs of diminishing returns as the batch sizes grows larger (e.g., grows from 256 to 512). We still consistently outperform the SimCLR baseline. 

\begin{figure}
    \centering
    \includegraphics[width=0.7\linewidth]{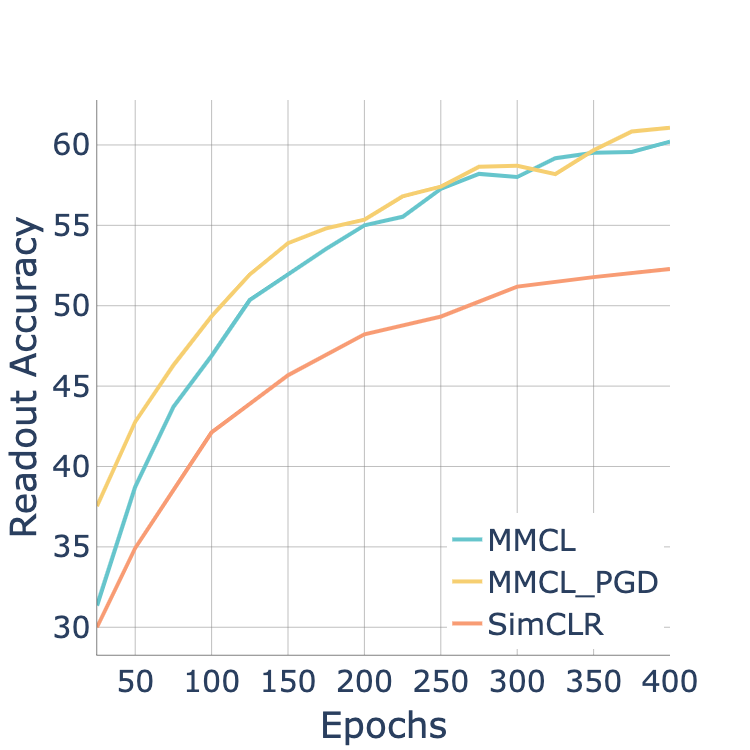}
    \caption{Convergence (CIFAR-100).}
    \label{fig:convergence_CIFAR100}
\end{figure}
\begin{table}[]
    \centering
\begin{tabular}{lcccr}
            \toprule
            CIFAR-100  & 64 & 128 & 256 & 512 \\
            \midrule
            SimCLR  & 60.8 & 65 & 66 & 66.32  \\
            DCL  & 63.2 & 66.2 &  67.5 & 67.71 \\
            HCL  & 66.46 & 68.37 &  \textbf{69} & \textbf{70.31}\\
            \midrule
            $\MMCL$  & 65.82  & \textbf{68.75} & \textbf{68.81} & \textbf{70.7}\\
            \bottomrule
            \end{tabular}
    \caption{Accuracy (in \%) against the batch size on CIFAR-100.}
    \label{tab: batchsize_cifar100}
\end{table}

\noindent\textbf{Convergence}
In Figure~\ref{fig:convergence} of the main paper, we show the convergence analysis of training on STL-10. In Figure~\ref{fig:convergence_CIFAR100} we add a similar analysis for STL-10. Similar to results on STL-10, we see that the MMCL variants converge to superior performances more rapidly than the baseline. 

\noindent\textbf{Kernels with SimCLR:}
The loss used in vanilla SimCLR (and subsequently other contrastive learning approaches) is technically equivalent to using an RBF kernel, and one could perhaps choose other similarity kernels, e.g., tanh. To quantitatively evaluate this intuition and the impact of such a choice, we used a tanh kernel in SimCLR, and experimented with 5 different bandwidths. The results (Table~\ref{tab: kernel_simclr}) suggest that while there can be some improvements to SimCLR via selecting other kernel similarities, our proposed formulation demonstrates significantly better performance. This is because our proposed kernel SVM formulation produces a classification margin in an RKHS where the support set is automatically sparse and weighted, which to the best of our knowledge is difficult to be argued for in the SimCLR setup.

\noindent\textbf{Using an Exact SVM Solver:}
Our initial approach was to use an exact SVM solver (using Qpth~\cite{amos2017optnet}) to estimate the max-margin hyperplane, however we found that our iterations were considerably slower (0.52 vs 3.07 for batch size 256).

\noindent\textbf{Intuition for Slack Penalty $C$:}
Note that in contrast to standard SVMs where the penalty $C$ is a trade off between the misclassification error and the margin, in our setup it has an even bigger role as the feature representation itself is evolving. Specifically, when using a larger $C$, the $\alpha$ weights on the misclassified points will be equal to $C$, and thus given our contrastive objective is using representations of data points linearly weighted by $\alpha$, the backbone neural network will be updated via gradients to minimize this loss -- thereby pushing these misclassified points out of the margin. Thus, in subsequent iterations, the focus of learning will be to increase the margin, as the number of misclassified points will be smaller. However, $C$ must also account for an estimate of false negatives, and thus a larger value of $C$ is perhaps not appropriate consistently. We found that an empirical evaluation of collisions is quite difficult, especially given that the network weights evolve over the epochs that collisions vanish as the contrastive loss drops. In Table~\ref{tab: C_FN_correction}, we attempted to define some heuristics to explicitly correct for false negatives (FN); our results show that when the penalty $C$ is small, correcting for the FNs do show minor benefits (e.g., C=50). However, for larger $C$, the network weights are perhaps updated very quickly to fix the FN misclassification error in the early epochs that FNs are absent in subsequent epochs. 

\begin{table}[]
    \centering
    \begin{tabular}{l|c}
    \toprule
        Method & Readout @ 200 \\
        \midrule
        SimCLR & 48.2 \\
        SimCLR tanh $\gamma=1$ & 50.5 \\
        SimCLR tanh $\gamma=0.5$ & 49.1 \\
        SimCLR tanh $\gamma=0.1$ & 49.9 \\
        SimCLR tanh $\gamma=2$ & 50 \\
        SimCLR tanh $\gamma=10$ & 49.9 \\
        MMCL & \textbf{55} \\
        \bottomrule
    \end{tabular}
    \caption{Use of Kernels with SimCLR. We see that while use of kernels does lead to an improvement for the SimCLR loss, our proposed formulation demonstrates significantly better performance }
    \label{tab: kernel_simclr}
\end{table}

\section{Other Details}
\noindent\textbf{PGD Solver:} We used nesterov accelerated PGD solver to get faster convergence. The value of $\alpha$ before each pre-training step is randomly initialized before the PGD iterations. We use a step size of 0.001 and optionally a step size of $\frac{1}{\|\Delta\|_2}$.

\noindent\textbf{ImageNet Experiments:} We use an initial learning rate of 1.2 and following SimCLR we use a warmup and an annealing scheme. Our best model uses $\sigma^2$ of 5. To keep hyperparameter search tractable, we obtain top-1 validation accuracy after pre-training for 10 epochs and finetuning for 1 epoch. Based on exeriments with ImageNet-100, we found the PGD variant to work slightly better than the INV variant (80.7\% vs 80.5\%) and so we use PGD variant for all ImageNet experiments. 

\noindent\textbf{CIFAR-100 and STL-10 Experiments:} We choose hyperparameters using the validation accuracy at 200 epochs based protocol mentioned in the main paper (Readout accuracy).

\noindent\textbf{UCF-101 Experiments} : We use a publicly available implementation for MoCo-v2 for video self supervised learning ~\cite{NEURIPS2020_3def184a}. All the hyperparameters are the same except $\sigma^2$ of 0.5, $C=1$, and learning rate of 1e-4. We use the PGD variant for experiments on UCF-101. 

\subsection{Limitations and Failure Cases}
The proposed approach has some minor overheads. While, we did not see any significant difference in training times (as shown in Figure~\ref{fig: compute_time}), they can be expensive atleast in the initial epochs of training. This is, for example, because, when the model is randomly initialized, the features produced are arbitrary and thus the PGD iterations in MMCL would need longer cycles to finish. However, we see that this trend is only for the initial few epochs and training speed improves over time. We could also perhaps use a different and better solver for box-constrained optimization that leads to much faster convergence (for example,~\cite{kiSrDh12} or OSQP~\cite{osqp-gpu})

\subsection{Potential Negative Impact}
Our paper makes a fundamental contribution to the field of self-supervised and contrastive learning. We do not perceive any obvious negative social impacts of this work. In fact, self-supervision naturally helps to avoid certain biases that can emerge from labeling biases. That said, training self-supervised learning methods on already biased datasets could be harmful and potentially lead to the biases being inherited by the models during the fine-tuning stage.

\bibliography{mmcl}
\end{document}

%% file: AAAI2022/definitions.tex
\newcommand{\vx}{x}%{\bm{x}}
\newcommand{\px}{x^+}%{\bm{x}^+}
\newcommand{\nx}{x^{-}} %{\bm{x}^-}
\newcommand{\vz}{\bm{z}}
\newcommand{\pz}{\bm{z}^+}

\newcommand{\vy}{\bm{y}}
\newcommand{\vw}{\bm{w}}
\newcommand{\valpha}{\bm{\alpha}}

\newcommand{\vone}{\bm{1}}
\newcommand{\one}{\bm{1}}
\newcommand{\tone}{\bm{1}^{\!\top}}

\newcommand{\dataset}{\mathcal{D}}

\newcommand{\kxx}{k_{xx}}
\newcommand{\kxy}{\bm{k}_{xY}}
\newcommand{\Kyy}{\bm{K}_{YY}}

\newcommand{\mDel}{\bm{\Delta}}
\newcommand{\nY}{Y^{-}}%{\bm{Y}^-}

\newcommand{\mX}{{X}}
\newcommand{\mY}{{Y}}
\newcommand{\pZ}{{Z}^{+}}
\newcommand{\nZ}{{Z}^{-}}
\newcommand{\nmZ}{{\bm{Z}}^{-}}
\newcommand{\batch}{B}

\newcommand{\kernel}{\mathcal{K}}
\newcommand{\Kernel}{\bm{K}}
\newcommand{\trans}{\mathcal{T}}
\newcommand{\pX}{{X}^+}
\newcommand{\nX}{{X}^-}
\newcommand{\sX}{{X}}

\newcommand{\half}{\tfrac{1}{2}}

\newcommand{\ft}{\bm{f}_\theta}

\newcommand{\hinge}[1]{\left[#1\right]_+}
\newcommand{\enorm}[1]{\left\|#1\right\|}
\newcommand{\inv}[1]{{#1}^{\!-\!1}}

\newcommand{\set}[1]{\left\{#1\right\}}

\newcommand{\ip}[2]{\langle #1, #2 \rangle}

\newcommand{\reals}[1]{\mathbb{R}^{#1}}

\newcommand{\spd}[1]{S^{#1}_{++}}

\DeclareMathOperator{\dist}{sim}
\DeclareMathOperator*{\MMCL}{MMCL}
\DeclareMathOperator*{\INVMMCL}{MMCL_{INV}}
\DeclareMathOperator*{\PGDMMCL}{MMCL_{PGD}}

\DeclareMathOperator*{\argmin}{arg\,min}

\newtheorem{prop}{Proposition}%[section]